\documentclass{ieeetj}
\usepackage{cite}
\usepackage{amsmath,amssymb,amsfonts}
\usepackage{graphicx,color}
\usepackage{textcomp}
\usepackage{xcolor}
\usepackage{hyperref}
\usepackage{bm}
\usepackage{dsfont}
\hypersetup{hidelinks=true}
\usepackage{algorithm}
\usepackage{algpseudocode}
\usepackage[capitalise]{cleveref}
\def\BibTeX{{\rm B\kern-.05em{\sc i\kern-.025em b}\kern-.08em
    T\kern-.1667em\lower.7ex\hbox{E}\kern-.125emX}}
\AtBeginDocument{\definecolor{tmlcncolor}{cmyk}{0.93,0.59,0.15,0.02}\definecolor{NavyBlue}{RGB}{0,86,125}}

\newcommand{\pderiv}[2]{\frac{\d{#1}}{\d{#2}}}

\newcommand{\grad}{\nabla }
\renewcommand{\d}{\partial}

\newcommand{\R}{\mathbb{R}}
\renewcommand{\Pr}{ \mathbb P }
\newcommand{\EXP}{\mathbb{E}}
\newcommand{\eqdstn}{\stackrel{d}{=}}
\DeclareMathOperator{\tr}{tr}

\newcommand{\MS}{\text{MS}}
\newcommand{\CE}{\text{CE}}

\newtheorem{definition}{Definition}

\newtheorem{proposition}{Proposition}
\newtheorem{lemma}{Lemma}

\def\OJlogo{\vspace{-4pt}$<$Society logo(s) and publication title will appear here.$>$}
\def\seclogo{\vspace{10pt}$<$Society logo(s) and publication title will appear here.$>$}

\def\authorrefmark#1{\ensuremath{^{\textbf{#1}}}}

\begin{document}
\receiveddate{XX Month, XXXX}
\reviseddate{XX Month, XXXX}
\accepteddate{XX Month, XXXX}
\publisheddate{XX Month, XXXX}
\currentdate{XX Month, XXXX}
\doiinfo{XXXX.2022.1234567}

\markboth{}{Author {et al.}}

\title{An exponential mechanism based on quadratic approximations for fine-tuning machine learning models with privacy guarantees\thanks{ORNL copyright}}

\author{Hoang Tran\authorrefmark{1},
Jorge Ramirez\authorrefmark{2},
Jiayi Wang\authorrefmark{3},
Alberto Bocchinfuso\authorrefmark{4},\\
Christopher Stanley\authorrefmark{3},
M. Paul Laiu\authorrefmark{1}
}
\affil{Computer Science and Mathematics Division, Oak Ridge National Laboratory, Oak Ridge, TN, 37831, USA}
\affil{Departamento de Matem\'aticas, Universidad Nacional de Colombia, Medell\'in, Colombia.}
\affil{Computational Science and Engineering Division, Oak Ridge National Laboratory, Oak Ridge, TN, 37831, USA}
\affil{HPC Department, Cineca, 40033 Casalecchio di Reno, Bologna, Italy}
\corresp{Corresponding author: Hoang Tran (email: tranha@ornl.gov).}
\authornote{This manuscript has been authored by UT-Battelle, LLC, under
contract DE-AC05-00OR22725 with the US Department of Energy (DOE).
The US government retains and the publisher, by accepting the article for
publication, acknowledges that the US government retains a nonexclusive,
paid-up, irrevocable, worldwide license to publish or reproduce the published
form of this manuscript, or allow others to do so, for US government purposes.
The DOE will provide public access to these results of federally sponsored
research in accordance with the DOE Public Access Plan \href{https://www.ieee.org/publications/services/thesaurus.html}{(http://energy.gov/downloads/doe-public-access-plan)}.}

\begin{abstract}
Fine-tuning adapts a pretrained machine learning model to a small, sensitive dataset, but this process risks memorizing individual new data points, making the model vulnerable to adversaries who seek to extract sensitive information. In this work, we develop a randomized algorithm based on the exponential mechanism for fine-tuning while ensuring differential privacy. Our key idea is to construct a simple utility function that combines a local quadratic approximation of the pretrained model with information from the new dataset.
{The resulting exponential mechanism admits exact sampling from a multivariate normal distribution in closed form.
We establish theoretical privacy guarantees, sensitivity bounds, and accuracy estimations for our method. 
We further introduce a random-projection strategy that makes the approach scalable to high-dimensional models. Numerical experiments on the MNIST benchmark and the MIMIC clinical dataset demonstrate competitive performance against existing differentially private fine-tuning techniques.
}
\end{abstract}



\maketitle

\section{Introduction}


Fine-tuning is an important process in machine learning that involves adapting a pretrained neural network to a new, often related, task. Given that modern neural networks are typically trained on large datasets and then repurposed for specialized tasks, fine-tuning enables models to quickly converge and achieve high accuracy without the need for extensive retraining from scratch. However, this process is not without its risks. Fine-tuning on sensitive or proprietary data can inadvertently expose private information contained within the new dataset, thereby posing significant privacy concerns.

In recent years, differential privacy (DP) \cite{10.1007/11681878_14,10.1007/978-3-540-79228-4_1} has emerged as a key framework for safeguarding sensitive information in the training and fine-tuning of deep neural networks. At its core, DP is enforced by adding calibrated random noise, often via mechanisms such as the Laplace or Gaussian mechanisms, to bound the influence of any individual training datapoint on the algorithm’s output. This approach guarantees that the presence or absence of one record cannot be reliably inferred, even by adversaries with auxiliary information. There are four main locations where such noise can be injected, leading to four broad categories of approaches: input perturbation \cite{10.1007/978-3-319-67786-6_6}, output perturbation \cite{lecuyer19,ijcai2019p660}, objective perturbation \cite{Phan_Wang_Wu_Dou_2016, 439b31d91ddd43cea2265ce6e770a5e5, Phan2017AdaptiveLM, Iyengar2019TowardsPD}, and gradient perturbation \cite{10.1145/2976749.2978318, 10.1145/3219819.3220076, yunot}. 

The most well-known DP technique is Differentially Private Stochastic Gradient Descent (DP-SGD) \cite{10.1145/2976749.2978318}, which employs the Gaussian mechanism to add noise to the gradients during training {based on a privacy budget}, thereby obscuring the influence of individual datapoints. Despite its success, DP-SGD often incur high computational overhead due to the need for per-sample gradient computation, clipping and noise addition, leading to higher training time and memory usage. In addition, the privacy budget scales with the number of training epochs. Therefore, the accumulated privacy cost can grow quickly over long training processes, making it difficult to train the models without exhausting the budget. While DP-SGD performs reasonably well on large, homogeneous datasets such as MNIST and CIFAR, it can struggle in sensitive domains like healthcare or finance, where datasets are typically small, imbalanced, and high-dimensional \cite{ziller2021medical,mohammadi2025differential,shen2021analysis}. 








An alternative approach is the Exponential Mechanism (ExpM) \cite{4389483}, where the model parameters are sampled from a probability distribution biased toward higher-utility models. Unlike DP-SGD, ExpM can provide privacy guarantees in a \textit{single} randomized selection, avoiding both the significant computational overhead and the progressive privacy degradation associated with repeated training steps. 
This makes ExpM particularly attractive in resource-constrained settings such as healthcare and mobile applications where data are scarce or training must remain fast and efficient. However, the practical adoption of ExpM has been limited due to two major obstacles. First, sampling from the target distribution is generally intractable due to the high–dimensional continuous parameter space, effectively restricting ExpM to multinomial sampling on finite or very simple domains \cite{10.1145/1835804.1835868, doi:10.1137/1.9781611973105.101,10.14778/3236187.3236202}. Second, applying ExpM requires a provable sensitivity bound for the chosen utility or loss function, which is often challenging to derive in practice, particularly for complex models with nonlinear objectives \cite{10.5555/3157096.3157204}. 
{Recent work~\cite{bridges2024normalizing} has investigated the use of normalizing flows (NFs) to approximate sampling from the target distribution induced by ExpM, suggesting a viable route toward scaling
ExpM to deep learning models. However, as also discussed
in~\cite{bridges2024normalizing}, the results are primarily empirical and do not constitute a rigorous privacy guarantee. Establishing one would require showing that the NF approximation to the target
distribution is sufficiently accurate, yet quantifying this approximation error remains an open problem. Moreover, NFs introduce practical challenges of their own: both the architecture and training hyperparameters require careful tuning, which could greatly affect the performance.
}



{
In this paper, we revisit ExpM in a different setting---model fine-tuning rather than training from scratch---and show that the aforementioned challenges can be fully addressed. The key insight is that fine-tuning operates in a regime where the loss landscape admits a tractable local approximation. Specifically, the fine-tuning dataset is typically drawn from a distribution close to that of the pretraining data, implying that the pretrained parameters are expected to be close to a minimum of the fine-tuning loss. The loss landscape around the pretrained parameters can therefore be accurately approximated by a quadratic form, making the resulting utility function amenable to both exact sampling and rigorous analysis---even though the global loss landscape remains highly complex due to the depth and nonlinearity of the model.
}


Our contributions in this paper are summarized below.

\begin{itemize}
\item We propose \textbf{ExpM-Quad}, an \textit{exponential mechanism} for fine-tuning machine learning models based on a \textit{quadratic approximation} of the loss function near the pretrained parameters.
\item 
We provide a sharp sensitivity estimate to support our framework and enable rigorous guarantees for both privacy and accuracy. Additionally, we derive a practical strategy for empirically estimating the sensitivity bound from training data.
{
\item To overcome the high sensitivity and computational cost of sampling in the full parameter space, we introduce a randomized projection, i.e., sketching, strategy that operates in a lower-dimensional subspace, making ExpM-Quad scalable to large models.
}
\item We conduct extensive numerical experiments on a regression problem, then the MNIST benchmark and the MIMIC clinical dataset. Our evaluation highlights the advantages of ExpM-Quad, demonstrating competitive performance relative to SGD and DP-SGD baselines.
\end{itemize}

The remainder of the paper is organized as follows. Section~\ref{sec:diffPriv} introduces the necessary preliminaries. We then develop ExpM-Quad through three components: (1) a utility function based on local quadratic approximation of the loss (Section~\ref{sec:util_func}); (2) sharp estimates of sensitivity bound for common regression and classification losses (Section~\ref{sec:application}); (3) sampling strategy to provide fine-tuned model parameters (Section~\ref{sec:sampling}). In Section~\ref{sec:projection}, we discuss the randomized projection and sampling method for large models with high dimensional parameters. The accuracy of ExpM-Quad is established in Section~\ref{sec:accuracy}. Section~\ref{sec:experiment} reports the experimental results. Finally, Section~\ref{sec:conclusion} concludes the paper.

\section{Differential privacy and the exponential mechanism}
\label{sec:diffPriv}

We consider the problem of DP for supervised learning over a dataset of input-target pairs $D=\{(x_i,y_i)\}_{i=1}^{|D|} \subset X \times Y$. 
Let $f(x,\theta)$ denotes a machine learning model parameterized by $\theta \in \mathbb{R}^p$ that maps input $x \in X$ to the corresponding target $y \in Y$.
Our goal is to find a model $x \mapsto f(x,\theta)$ trained on $D$ that is both \textit{accurate} and \textit{differentially private}, provided pretrained model parameters $\theta^*\in \mathbb{R}^p$ from a related dataset.


By accuracy, we refer to minimizing an empirical loss function over $D$. DP, on the other hand, ensures that 
the distribution of learned parameters is insensitive to the presence or absence of any single datapoint, thereby protecting individual information from inference attacks.

To achieve DP, the model parameters are obtained as samples $\theta \sim \mathcal{M}(D)$ from a \textit{randomized algorithm} $\mathcal{M}$ that defines a probability distribution over $\mathbb{R}^p$ for any $D$ in the universe $\mathcal{X}$ of possible datasets. The goal is for $\mathcal{M}$ not to be very sensitive to small changes in $D$. Specifically, for any $\varepsilon,\delta>0$, we say that $\mathcal{M}$ is  \textit{$(\varepsilon,\delta)$-DP} if 
\begin{equation}\label{eq:defDP}
\Pr[\mathcal{M}(D) \in S] \leq e^{\varepsilon} \cdot \Pr[\mathcal{M}(D') \in S]+\delta,
\end{equation}
for any measurable subset of outputs $S \subset \mathbb{R}^p$ and for any two ``adjacent'' datasets $D, D'\in\mathcal{X}$, namely two datasets differing in exactly one individual datapoint. Here, $\varepsilon$ controls the privacy loss (with smaller $\varepsilon$ implying stronger privacy), while $\delta$ allows for a small probability of failure. If $\delta=0$, we say that $\mathcal{M}$ is $\varepsilon$-DP.

The \textit{exponential mechanism} is a typical method for attaining $\varepsilon$-DP. It introduces a \textit{utility function} $U$, often derived from an objective or loss function, and defines the randomized algorithm $\mathcal{M}_U$ so that higher probability is assigned to areas of parameter space with a high utility value $U$. 

Formally, let $B\subseteq \mathbb{R}^p$ denote the region of possible values of the model parameters, and suppose $U: \mathcal{X} \times B \to \mathbb{R}$ quantifies the quality of the parameter $\theta \in B$ on dataset $D \in \mathcal{X}$. An {ExpM} $\mathcal{M}_U$ with utility $U$ can be designed to select $\theta$ with probability proportional to
\begin{align}
\Pr[\mathcal{M}_U(D) = \theta] \propto \exp\left( \frac{\varepsilon \, U(D, \theta)}{2 \overline{\Delta U}} \right) \mathds{1}_{B}(\theta),
\label{eq:ExpM}
\end{align}
where $\mathds{1}_B$ is the indicator function over $B$, and $\overline{\Delta U}$ is an upper bound of $\Delta U$, which denotes the \textit{sensitivity} of $U$, i.e.,
  \begin{align}
\label{eq:ExpM_sens}
    \Delta U = \max_{\theta \in B} \max_{\substack{D, D' \in \mathcal{X}\\ D, D'\text{\ adjacent } }} |U(D, \theta) - U(D', \theta)|.
  \end{align}
It can be proven that this ExpM $\mathcal{M}_U$ satisfies the $\varepsilon$-DP criterion for any $\varepsilon \geq 0$, see \cite{10.1561/0400000042}. 

For implementing an ExpM, one needs to address three challenges: defining a utility function $U$, finding a bound $\overline{\Delta U}$ for the utility sensitivity, and designing a sampling method for \cref{eq:ExpM}. These are discussed in the next three sections. 

\section{A utility function for fine-tuning}
\label{sec:util_func}

We are interested in providing a fine-tuned model that is DP over dataset $D\in\mathcal{X}_N$ from an existing pretrained model $f(\cdot,\theta^*)$.
Here the universe of datasets $\mathcal{X}_N$ includes datasets of size larger than $N$.


Namely, we aim for a randomized algorithm $\theta = \mathcal{M}(D)$ that provides an accurate fine-tuned model with respect to a loss function of the form
\begin{equation}\label{eq:defLoss}
    L(D,\theta) = \frac{1}{n} \sum_{i=1}^{n} \ell(y_i, f(x_i,\theta))
\end{equation}
for $D=\{(x_i,y_i)\}_{i=1}^n \in \mathcal{X}_N$, $n=|D| > N$, $\theta \in B$, where the per-datapoint loss $\ell$ is a positive and smooth function.

Our key assumption is that the fine-tuned parameter $\theta$ can be found in a neighborhood of the pretrained parameter $\theta^*$. Specifically, for any given $D \in \mathcal{X}_N$, we assume that $L(\cdot,D)$ is well-approximated by a quadratic function $\tilde{L}(\cdot,D)$ around $
\theta^*$; that is, the parameter $\theta$ yielding small values of ${L}(\theta,D)$ also gives small values for $\tilde{L}(\theta,D)$. This assumption is well justified in the fine-tuning setting where the new dataset is typically drawn from a distribution highly similar to that of the dataset used in the pretraining. 

We define the approximate loss function as
\begin{equation}\label{eq:defLtilde}
\begin{split}
    \tilde{L}(D,\theta) &= (\theta - \theta^*)^\top g(D)\\ 
    &+ \frac{1}{2} (\theta - \theta^*)^\top \boldsymbol{H}(D) (\theta - \theta^*)\\
    &+ \frac{\lambda}{2} |\theta - \theta^*|^2,
\end{split}\end{equation}
where $g(D) = \nabla_\theta[L(D,\cdot)](\theta^*)$ is the gradient of the loss function with respect to $\theta$, evaluated at $\theta^*$, $\bm{H}(D)$ denotes an approximation to the corresponding Hessian matrix $\bm{\nabla^2}_\theta[L(D,\cdot)](\theta^*)$, and the last term regularizes the approximation with regularization parameter $\lambda>0$ penalize $\theta$'s that are away from $\theta^*$.


Note that from the definition of the loss function in \cref{eq:defLoss}, $g$ can be written as an average of the per-datapoint gradients
\begin{align}
    g(D) &= \frac{1}{n} \sum_{i=1}^{n} g(x_i,y_i), \label{eq:gofD}\\
    g(x,y)&=
      (\bm{\grad}_\theta f(x,\theta^*))^\top \grad_f \ell(y,f(x,\theta^*)),\label{eq:gxy}
\end{align}
where, for convenience, we use $g(x,y)$ to denote the gradient $ g(\{(x,y)\})$ on a dataset of a single datapoint. Also, note that the boldface in $\bm{\grad}_\theta f$ connotes the fact that this term is the Jacobian matrix. 

To arrive at a suitable approximation $\bm{H}(D) \approx \bm{\nabla^2}_\theta[L(D,\cdot)](\theta^*)$ of the Hessian matrix, we differentiate $g(x,y)$ in \cref{eq:gxy} to obtain {the exact Hessian}
\begin{equation}\label{eq:fullHessl}
    \bm{\grad^2}_\theta\ell(\cdot,f(\cdot,\theta)) =   
    (\bm{\grad}_\theta f)^\top \bm{\grad^2}_f \ell  \,\bm{\grad}_\theta f
    + \sum_k \bm{\grad^2}_\theta f^{(k)} \pderiv{\ell}{f^{(k)}}.
\end{equation}
In order to avoid the costly calculation of the Hessian matrices $\bm{\grad^2}_\theta f^{(k)}$ of the model with respect to all its parameters, we apply the Gauss-Newton approximation to \cref{eq:fullHessl} and propose
\begin{equation}\label{eq:GNHess}
    \bm{H}(x,y) = \bm{\grad}_\theta f(x,\theta^*)^\top \bm{\grad^2}_f \ell(y,f(x,\theta^*))  \bm{\grad}_\theta f(x,\theta^*).
\end{equation}
where, again, $\bm{H}(x,y)$ denotes $\bm{H}(\{(x,y)\})$. We thus complete the definition of $\tilde{L}$ in \cref{eq:defLtilde} with
\begin{equation}\label{eq:HofD}
    \bm{H}(D) = \frac{1}{n} \sum_{i=1}^{n} \bm{H}(x_i,y_i),
\end{equation}
which is symmetric and positive definite by construction.

\begin{definition}\label{def:expMech}
    Let $\varepsilon >0$, $R>0, \theta^* \in \R^p$ and denote the ball of radius $R$ around $\theta^*$ by
\begin{equation}
    B^{(p)}_{R,\theta^*} := \{\theta \in \R^p:  |\theta - \theta^*|\leq R\}.
\end{equation} 
Our \textit{exponential mechanism $\mathcal{M}_U(D)$ for fine-tuning} the model $f(\cdot,\theta^*)$ over datasets $D \in \mathcal{X}_N$ is given by \cref{eq:ExpM} with utility function
\begin{equation}\label{eq:ourU}
    U(D,\theta) := - \tilde{L}(D,\theta), \quad \text{for }  \theta \in B^{(p)}_{R,\theta^*},
\end{equation}
where $\tilde{L}$ is as \cref{eq:defLtilde}: $\lambda >0$ the regularization parameter, $g(D) = \nabla_\theta[L(D,\cdot)](\theta^*)$, and $\bm{H}(D)$ given by the Gauss-Newton approximation in \cref{eq:HofD,eq:GNHess} to the Hessian.
\end{definition}

\subsection{Sensitivity estimate}
In order to derive an upper bound $\overline{\Delta U}$ to the sensitivity in \cref{eq:ExpM_sens}, we write the utility function as
\begin{equation}
\label{eq:util2}
    U(D,\theta) =- \frac{\lambda}{2} |\theta - \theta^*|^2 
    + \frac{1}{n} \sum_{i=1}^{n} u((x_i,y_i),\theta),
\end{equation}
where $u$ is the per-datapoint utility, i.e., 
\begin{equation}\begin{split}
    u((x,y),\theta) &= -(\theta - \theta^*)^\top g(x,y) \\
    &- \frac{1}{2} (\theta - \theta^*)^\top \boldsymbol{H}(x,y) (\theta - \theta^*) \label{eq:defu}
\end{split}\end{equation}
for $(x,y) \in X \times Y$ and $\theta \in B^{(p)}_{R,\theta^*}$. Note that the regularization term in $U$ does not depend on the dataset $D$ and thus plays no role in the sensitivity. In this subsection, without loss of generalization, we drop this term in \cref{eq:util2}.  

Let $D =\{(x_i,y_i)\}_{i=1}^n \in \mathcal{X}_N,\, n> N$ and suppose $D'\in \mathcal{X}_N$ is an adjacent dataset obtained by adding a datapoint $(x_{n+1},y_{n+1})$ to $D$. Here, $(x_{n+1},y_{n+1})$ represents a potential member that need not belong to $D$ a priori. Then  
\begin{align}
    U(D',\theta) &= \frac{n U(D,\theta) + u((x_{n+1},y_{n+1}),\theta)}{n+1}, \nonumber\\
    |U(D',\theta)\!-\!U(D,\theta)| &= \frac{1}{n+1}\! \left|  u((x_{n+1},y_{n+1}),\theta) \! - \!U(D,\theta) \right| \nonumber\\
    \leq \frac{1}{N} & \left(|u((x_{n+1},y_{n+1}),\theta)| + |U(D,\theta)|\right).
    \label{eq:DeltaU0}
\end{align}

In order to globally bound \cref{eq:DeltaU0}, we introduce upper bounds $\overline{g},\overline{\bm{H}}>0$ for the norms of the per-datapoint gradient and approximate Hessian,
\begin{align}
    \max_{(x,y) \in X \times Y} |g(x,y)| & \leq \bar{g}, \label{eq:def_gbar}\\
    \max_{(x,y) \in X \times Y} |\bm{H}(x,y)| &\leq \overline{\bm H}; \label{eq:def_Hbar}
\end{align}
which, because of \cref{eq:gofD,eq:HofD}, also apply to $g(D)$ and $\bm{H}(D)$:
\begin{equation}
    |g(D)| \leq \overline{g}, \quad |\bm{H}(D)| \leq \overline{\bm H}
\end{equation}
for all $D \in \mathcal{X}_N$. Note that the maxima in \cref{eq:def_gbar,eq:def_Hbar} are taken over all possible input–output pairs and not merely over those contained in $D$.

With these upper bounds, applying the Cauchy-Schwartz inequality on \cref{eq:defu} then gives
\begin{equation}
   \max_{\substack{(x,y) \in X \times Y\\ \theta \in B^{(p)}_{R,\theta^*} }}
 |u((x,y),\theta)| \leq R ~ \overline{g} + \frac{1}{2}R^2 ~\overline{\bm H}
\end{equation}
for all $(x,y) \in X \times Y$ and $\theta \in B^{(p)}_{R,\theta^*}$. Since $U$ is an average, the same upper bound holds for $|U|$, allowing us to define the sensitivity bound by
\begin{align}
    \left|U(D,\theta) - U(D',\theta)\right| & \leq \frac{2 R \, \overline{g} + R^2 \, \overline{\bm H}}{N} 
    =: \overline{\Delta U}.\label{eq:DeltaU}
\end{align}
This same sensitivity bound in \cref{eq:DeltaU} can be obtained if we consider $D,D'\in \mathcal{X}_N$ adjacent in the ``replace-one'' sense, namely, $D$ and $D'$ differing only on their last input-output pairs, $(x_n,y_n) \neq (x'_n,y'_n)$. In this case, 
\begin{align*}
    \left|U(D,\theta) - U(D',\theta)\right| &= \frac{1}{n} \left|u((x_n,y_n),\theta) - u((x'_n,y'_n),\theta)\right|\\
    &\leq 2\max_{\substack{(x,y) \in X \times Y\\ \theta \in B^{(p)}_{R,\theta^*} }}
 |u((x,y),\theta)|. 
\end{align*}
See \cite{10.1145/3580305.3599561} for a discussion of different meanings of adjacency between datasets and its implication on DP.

To derive the upper bounds $\overline{g}$ and $\overline{\bm{H}}$, we start by assuming a bound $\overline{\bm{\grad}_\theta f}>0$ on the 2-norm of the Jacobian of the model at $\theta^*$ over inputs, 
\begin{equation}\label{eq:boundGrad}
    \max_{x \in X} |\bm{\grad}_\theta f(x,\theta^*)| \leq \overline{\bm{\grad}_\theta f}.
\end{equation}
Then, applying the Cauchy-Schwarz inequality on \cref{eq:gxy,eq:GNHess} leads to
\begin{align}
    \max_{(x,y)} |g(x,y)| &\leq \overline{\bm{\grad}_\theta f}
    \max_{(x,y)} |\grad_f \ell(y,f(x,\theta^*))|, \label{eq:maxg}\\
    \max_{(x,y)} |\bm{H}(x,y)|_F &\leq \overline{\bm{\grad}_\theta f}^2
    \max_{(x,y)} |\bm{\grad^2}_f \ell(y,f(x,\theta^*))|.
    \label{eq:maxH}
\end{align}
With notice that $\overline{\bm{\grad}_\theta f}$ can be estimated directly from the dataset, \cref{eq:maxg,eq:maxH} can be viewed as expressions of $\overline{g}$ and $\overline{\bm{H}}$ in terms of $\grad_f \ell$ and $\bm{\grad^2}_f \ell$.
 In the next section, we provide bounds for the norms of the gradient $\grad_f \ell$ and Hessian $\bm{\grad^2}_f \ell$ for specific loss functions.

\section{Sensitivity bounds for regression and classification models}
\label{sec:application}
We now proceed to specify the approximate Hessian matrix $\bm H$ and the corresponding bounds $\overline{g}, \overline{\bm H}$ for typical loss functions used in supervised learning.

We consider fine-tuning models for regression and classification tasks that use either mean-squared (MS) or cross entropy (CE) loss on the output space $Y \subset \R^m$. We denote these cases with the subscripts MS and CE throughout. Specifically, for the MS loss,
\begin{equation}\label{eq:lMS}
    \ell_{\MS}(y,f) = \frac{1}{m}|y-f|^2.
\end{equation}
For classification tasks with $m$ classes, the CE loss function composes the softmax $\sigma$ with the model output and computes the cross-entropy against $y$,
\begin{equation}
    \ell_{\CE}(y,f) = -y^\top \log(\sigma(f)), \label{eq:lCE}
\end{equation}
where the $\log$ in \cref{eq:lCE} is applied component-wise.

The following proposition provides global bounds for the gradients $\grad_f \ell$ and Hessian $\bm{\grad^2}_f \ell$ for each type of model.

\begin{proposition}\label{prop:boundls}
Let $m$ be the dimension of the output space, i.e., $Y\subset \R^m$. For the MS loss,
    \begin{align}
       \max_{(x,y) \in X \times Y} |\grad_f \ell_\MS(y,f(x,\theta^*))| &\leq \frac{2}{m} \overline{E}_{\theta^*}, \label{eq:boundglMS}\\
       \max_{(x,y) \in X \times Y} |\bm{\grad^2}_f \ell_\MS(y,f(x,\theta^*))| &\leq \frac{2}{m}, \label{eq:boundHlMS}
    \end{align}
where $\overline{E}_{\theta^*}$ is global bound of the model error, namely
\begin{equation}
    \max_{(x,y) \in X \times Y} |f(x,\theta^*) - y| \leq \overline{E}_{\theta^*}.
\end{equation}
For the CE loss,
\begin{align}
       \max_{(x,y) \in X \times Y} |\grad_f \ell_\CE(y,f(x,\theta^*))| &\leq \sqrt{2}, \label{eq:boundglCE}\\
       \max_{(x,y) \in X \times Y} |\bm{\grad^2}_f \ell_\CE(y,f(x,\theta^*))| &\leq \frac{1}{2} \label{eq:boundHlCE}.
    \end{align}
\end{proposition}
\begin{proof}
For the MS case, differentiating $\ell_{\MS}$ in \cref{eq:lMS} gives
    \begin{align}
    \grad_f \ell_{\MS}(y,f) &= \frac{2}{m}(f-y) \label{eq:glMS}\\
    \bm{\grad^2}_f \ell_\MS(y,f) &= \frac{2}{m} \bm{I} \label{eq:HlMS}
\end{align}
from which the bounds in \cref{eq:boundglMS,eq:boundHlMS} follow.

For classification with CE, both the data $y$ and the model output $\sigma(f)$ are probability vectors. Using this, the gradient and Hessian of $\ell_{\CE}$ are
\begin{align}    
    \grad_f \ell_{\CE}(y,f) &=  \sigma(f) -y, \label{eq:glCE}\\
    \bm{\grad^2}_f \ell_\CE(y,f) &= \text{diag}(\sigma(f)) - \sigma(f) \sigma(f)^\top. \label{eq:HlCE}
\end{align}
Since $|\sigma(f)|_1=|y|_1=1$, then $|\sigma(f)|^2 \leq 1$, $|y|^2 \leq 1$ and \cref{eq:boundglCE} follows from the triangle inequality. To obtain \cref{eq:boundHlCE}, consider the row-sum norm  $|\bm{\grad^2}_f \ell_\CE(y,f)|_\infty$. Using the fact that $\sigma(f)$ is a probability vector, we can write
\begin{align}
    |\bm{\grad^2}_f \ell_\CE|_{\infty} &= \max_i \big\{|\sigma_i - \sigma_i^2| + \sum_{i\neq j} |\sigma_i \sigma_j| \big\}\\
    &= \max_i \big\{ \sigma_i (1-\sigma_i) + \sigma_i \sum_{i\neq j} \sigma_j \big\}\\
    &=\max_i \big\{ 2 \sigma_i(1-\sigma_i) \big\} \leq \frac{1}{2}.
\end{align}
By symmetry, the same bound applies to the column-sum norm $|\bm{\grad^2}_f \ell_\CE|_{1} \leq \frac{1}{2}$. The operator norm can thus be bounded as
\begin{equation}
    |\bm{\grad^2}_f \ell_\CE| \leq \sqrt{|\bm{\grad^2}_f \ell_\CE|_1 \, |\bm{\grad^2}_f \ell_\CE|_\infty} \leq \frac{1}{2}
\end{equation}
proving \cref{eq:boundHlCE}.
\end{proof}

The per-datapoint gradients $g_\MS(x,y)$ and $g_\CE(x,y)$ are obtained by multiplying the gradients in \cref{eq:glMS,eq:glCE} times the Jacobian $\bm{\grad}_\theta f(x,\theta^*)$ as per \cref{eq:gxy}. Similarly, \cref{eq:HlMS,eq:HlCE} are inserted in \cref{eq:GNHess} to obtain the Gauss-Newton approximates $\bm{H}_\MS(x,y)$ and $\bm{H}_{CE}(x,y)$ to the full Hessians of $\ell$ respect to $\theta$. Combining these with the results of \cref{prop:boundls}, \cref{eq:maxg,eq:maxH}, and the bound for the general sensitivity in \cref{eq:DeltaU}, we get
\begin{align}
    \overline{\Delta U_\MS} &= \frac{2R}{m N}  \left(2 \overline{E}_{\theta^*} + R \overline{\bm{\grad}_\theta f} \right) \overline{\bm{\grad}_\theta f}, \label{eq:sens_MS}\\
    \overline{\Delta U_\CE} &= \frac{2R}{N}  \big(\sqrt{2} + \frac{R}{4} \overline{\bm{\grad}_\theta f} \big) \overline{\bm{\grad}_\theta f}. \label{eq:sens_CE}
\end{align}

\section{Sampling the exponential mechanism}
\label{sec:sampling}
The appeal of the exponential mechanism in \cref{def:expMech} is that sampling from it is very simple: since the utility function is a quadratic function on $\theta$, then $\mathcal{M}_U(D)$ is a Gaussian distribution on $\R^p$ restricted to $B^{(p)}_{R,\theta^*}$.  

To see this, we include the regularization term of \cref{eq:defLtilde} into the Hessian approximation, and define
\begin{equation}\label{eq:Hlambda}
    \boldsymbol{H}_\lambda(D) = \bm{H}(D) + \lambda \bm{I}.
\end{equation}
Then we complete the square in $U(D,\theta) = -\tilde{L}(D,\theta)$ to write the utility as
\begin{equation}
    U(D,\theta) = -\frac{1}{2}(\theta - \mu(D))^\top \boldsymbol{H}_\lambda(D) (\theta - \mu(D)) + c,
\end{equation}
where $c$ is a constant and $\mu$ is the \textit{mean vector},
\begin{equation}\label{eq:defmu}
    \mu(D) = \theta^* - \boldsymbol{H}_\lambda(D)^{-1}g(D).
\end{equation}
Then, in accordance with \cref{def:expMech} and \cref{eq:ExpM}, the exponential mechanism has a truncated Gaussian distribution 
\begin{equation}\label{eq:distMU}
    \mathcal{M}_U(D) \eqdstn \text{Normal}\left(\mu(D), \frac{2 \overline{\Delta U}}{\varepsilon} \bm{H}_\lambda(D)^{-1}\right)\mathds{1}_{B^{(p)}_{R,\theta^*}}.
\end{equation}
A sample $\theta = \mathcal{M}_U(D)$ can thus be obtained by rejection sampling, that is, sampling $\theta$ from the $p$-dimensional Gaussian distribution in \cref{eq:distMU} and rejecting whenever $|\theta - \theta^*| > R$. Alternatively, Gibbs sampling is also an appealing approach for this distribution because it can exploit 1D truncated Gaussian to generate coordinate-wise updates that respect the constraint without rejection. 

The choice of $R$ requires careful consideration. If $R$ is too small, the rejection sampling suffers from extremely high rejection rates, while Gibbs sampling may exhibit slow mixing due to the tight truncation. Conversely, choosing $R$ too big leads to an unnecessarily large sampling region in which the quadratic approximation may no longer capture the true loss function. The value of the mean $\mu(D)$ in \cref{eq:defmu} suggests a useful rule of thumb for picking the radius $R$ of the feasibility region. Namely, in order for $\mu(D) \in B^{(p)}_{R,\theta^*}$, we must have $R \geq |\boldsymbol{H}_\lambda(D)^{-1}g(D)|$. Consider the following inequality
\begin{equation}\label{eq:bounds_Hlambda}
    \frac{1}{\lambda + \overline{\bm H}}
    \leq \frac{1}{\lambda + |\bm{H}|}
    \leq |\bm{H}_{\lambda}^{-1}|
    \leq \frac{1}{\lambda + e_{\min}(\bm{H})}
    \leq \frac{1}{\lambda},
\end{equation}
where $e_{\min}(\bm{H}) > 0$ is the smallest eigenvalue of $\bm H$. Then $|\mu - \theta^*| \leq \overline{R_\lambda}$ with
\begin{equation}\label{eq:defRbar}
    \overline{R_\lambda} := \frac{\overline{g} }{\lambda + e_{\min}(\bm{H})},
\end{equation}
and we can set $R > \overline{R_\lambda}$. 

For the rejection method, the efficiency depends on the probability of rejection, which can be bounded by Markov's inequality as 
\begin{align}
    \Pr[|\theta - \theta^*| > R] &\leq 
    \frac{\EXP[|\theta - \mu|^2] + |\mu - \theta^*|^2}{R^2} \nonumber\\
    &= \frac{\frac{2\overline{\Delta U}}{\varepsilon} \tr(\bm{H}_\lambda^{-1})+|\boldsymbol{H}_\lambda^{-1}g|^2}{R^2} \nonumber\\
    &=\frac{2\tr(\bm{H}_\lambda^{-1})}{\varepsilon N}\left(\frac{2 \overline{g}}{R} +  \overline{\bm{H}}\right)+ \frac{|\boldsymbol{H}_\lambda^{-1}g|^2}{R^2} \nonumber\\
    &\leq \frac{2p\overline{R_\lambda}}{N\varepsilon}\left(\frac{ 2 }{R} +  \frac{\overline{\bm{H}}}{\overline{g}}\right) + \left(\frac{\overline{R_\lambda}}{R}\right)^2,
    \label{eq:rejProb}
\end{align}
where we have used the fact that $\EXP[|\theta - \mu|^2] = \frac{2\overline{\Delta U}}{\varepsilon}  \tr(\bm{H}_\lambda^{-1})$ together with \cref{eq:bounds_Hlambda}, \cref{eq:defRbar} and
\begin{equation}
    \frac{p}{(\lambda + \overline{\bm{H}})} \leq
     \tr(\bm{H}_\lambda^{-1}) \leq
    \frac{p}{\lambda+e_{\min}(\bm{H})}.
\end{equation}
The bound in \cref{eq:rejProb} highlights the tradeoffs for the efficiency of our method in terms of how close $\theta^*$ is to the minimizer of the fine-tuning loss function. 

The complete method is summarized in \cref{alg:main_proj}.

\begin{algorithm}[ht!]
\caption{Exponential mechanism based on quadratic approximations (ExpM-Quad)} \label{alg:main_proj}
\textbf{Inputs:}
\begin{itemize}
\item Pre-trained model $f(\cdot,\theta^*):X \to Y$ with parameters $\theta^* \in \R^p$.
\item Private dataset $D = \{(x_i,y_i)\}_{i=1}^{|D|} \subset X\times Y$.
\item Differentiable per-datapoint loss function $\ell: Y \times Y \to [0,\infty)$.
\item Privacy budget $\varepsilon > 0$.
\item Global bounds $\overline{\bm{\grad}_\theta f}$ in \cref{eq:boundGrad} and bounds for $\nabla_f \ell, \bm{\nabla^2}_f \ell$ in \cref{prop:boundls}.
\end{itemize}
\textbf{Outputs:} Differentially private fine-tuned parameter ${\theta}$.
\\
\textbf{Procedure:}
\begin{algorithmic}[1]
\State Compute bounds $\overline{g}$, $\overline{\bm H}$ in \cref{eq:def_gbar,eq:def_Hbar} and sensitivity bound $\overline{\Delta U}$ in \cref{eq:DeltaU}.
\State Use \cref{eq:gofD,eq:gxy} to compute the gradient of the loss at $\theta^*$, $g(D)$.
\State Use \cref{eq:GNHess,eq:HofD} to compute the Gauss-Newton approximation $\bm{H}(D)$ to the Hessian matrix. 
\State Choose the regularization parameter $\lambda > \overline{\bm H}$ and set $\bm{H}_\lambda(D) \leftarrow \bm{H}(D) + \lambda \bm{I} $.
\State Compute $\overline{R}_\lambda$ in \cref{eq:defRbar} and use it as a reference for selecting $R$. 
\State Compute $\bm{H}_\lambda(D)^{-1}$ and $\mu(D)$ in \cref{eq:defmu}.
\State Draw a sample from the truncated Gaussian distribution
 $$\theta \sim \text{Normal}\left(\mu(D), \frac{2 \overline{\Delta U}}{\varepsilon} \bm{H}_\lambda(D)^{-1}\right)\mathds{1}_{B^{(p)}_{R,\theta^*}}$$
 using rejection or Gibbs sampling. 
 \State \Return $\theta$
\end{algorithmic}
\end{algorithm}

\section{Extension to high-dimensional models}
\label{sec:projection}
For machine learning models with high-dimensional parameters, i.e., $p$ is large, a straightforward application of Algorithm \ref{alg:main_proj} can be inefficient. The difficulty is twofold. First, in high-dimensional spaces, candidate samples drawn directly from the multivariate Gaussian will tend to lie extremely far away from $\mu$, and hence from $\theta^*$. The second difficulty is computational. Evaluating and exploiting local curvature information via the Gauss-Newton approximation of the Hessian in \cref{eq:GNHess} becomes increasingly expensive at scale since it is fundamentally tied to a $p\times p$ structure. Concretely, one either (i) stores and applies the dense $p\times p$ Gauss-Newton matrix, which requires $O(p^2)$ memory and operations, or (ii) avoids forming it but incurs many evaluations and multiplications by the Jacobian $\bm{\nabla}_\theta f(\cdot,\theta^*)$. Moreover, sampling from the Gaussian distribution in \cref{eq:distMU} entails applying $\bm{H}_\lambda(D)^{-1/2}$ to a standard normal vector. This requires either an explicit factorization of $\bm{H}_\lambda(D)$ (costing $O(p^3)$ time and $O(p^2)$ memory), or iterative matrix-function methods that use many Hessian--vector products.

To address these challenges, we propose a projection-based variant of the ExpM-Quad that restricts sampling to a randomly selected low-dimensional subspace. This restriction improves computational tractability, yields more efficient exploration of the parameter space, and retains rigorous DP guarantees. To formalize this idea, we introduce a \emph{restricted utility function} obtained by projecting the high-dimensional parameter space $\R^p$ onto an affine subspace $\R^{\tilde{p}}$ with $\tilde{p} \ll p$. Let \( \bm{A} \in \mathbb{R}^{p \times \tilde{p}} \) be a projection matrix with orthonormal columns. For any low-dimensional parameter vector \( \xi \in B^{(\tilde{p})}_{R,0} = \{\xi \in \R^{\tilde{p}}: |\xi| \leq R\} \), we map it back to the original parameter space via
\begin{equation}
    \theta_{\bm A}(\xi) = \theta^* + \bm{A} \xi \in B^{(p)}_{R,\theta^*}
\end{equation}
Based on this transformation, we define the \textit{restricted utility function} $U_{\bm A}: \mathcal{X}_N \times B^{(\tilde{p})}_{R,0} \to \R$ as 
\begin{align}
    U_{\bm A}(D,\xi) &= U(D,\theta_{\bm A}(\xi)) = - \tilde{L}(D,\theta_{\bm A}(\xi)) \nonumber\\
    &= - \xi^\top g_{\bm{A}}(D) - \frac{1}{2} \xi^\top \bm{H}_{\lambda, \bm{A}}(D) \xi \label{eq:defUA}
\end{align}
in terms of the $\tilde{p}$-dimensional projections of the gradient and approximated Hessian,
\begin{equation}\label{eq:defgHA}
    g_{\bm{A}}(D) = \bm{A}^\top g(D),\quad
    \bm{H}_{\lambda, \bm{A}}(D) = \bm{A}^\top \bm{H}_\lambda(D) \bm{A}.
\end{equation}

This formulation restricts the sampling domain of the ExpM to  \( \R^{\tilde{p}} \), centered at \( \theta^* \). The projection matrix \( \bm{A} \) can be chosen randomly -- for example, by drawing from a standard Gaussian distribution and normalizing the columns -- or constructed using problem-specific information, such as the top eigenvectors of the Hessian on a reference dataset $D_0$. 
Note that the transformation $\bm{A}$ does not affect the upper bounds $\overline{g}$ or $\overline{\bm H}$ and we thus have $\overline{\Delta U_{\bm{A}}} = \overline{\Delta U}$, which bounds the restricted sensitivity.

Given $\bm{A}$, the \textit{restricted exponential mechanism} provides $\theta = \theta_{\bm A}(\xi)$ from a sample  
\begin{equation}\label{eq:distMUA}
    \xi \!\sim\! \mathcal{M}_{U_{\bm A}}(D)\! :=\! \text{Normal}\left(\!\mu_{\bm A}(D), \frac{2 \overline{\Delta U}}{\varepsilon} \bm{H}_{\lambda,\bm{A}}(D)^{-1}\!\right)\mathds{1}_{B^{(\tilde{p})}_{R,0}}
\end{equation}
where the mean is given by
\begin{equation}
    \mu_{\bm A}(D) = \bm{H}_{\lambda,\bm{A}}(D)^{-1} g_{\bm A}(D).
\end{equation}
In a similar spirit to Algorithm 1, we can use $\mu_{\bm A}(D)$ as a reference for picking radius $R$ such that the mean lies inside the feasibility region. Note that $\mu_{\bm A}(D)$ depends on the choice of $\bm A$, and consequently, so does this criterion for choosing $R$.

The restricted algorithm offers two key advantages. First, it reduces the effective sampling dimension: instead of drawing parameters in $\mathbb{R}^p$, the mechanism samples $\xi \in \mathbb{R}^{\tilde{p}}$ with $\tilde{p} \ll p$, so gradient and curvature computations involve only matrices of size $\tilde{p}\times\tilde{p}$. In particular, the restricted utility in \cref{eq:defUA}  depends only on the projected gradient $ g_{\bm A}$ and the compressed Hessian $\bm{H}_{\lambda,\bm{A}}$ in \cref{eq:defgHA}. After a one-time construction of these projections, all subsequent computations are carried out entirely at the reduced dimension. Because directions orthogonal to the subspace are frozen, the samples naturally remain closer to the pretrained parameters. Second, the approach avoids the high costs of forming the full $p\times p$ Gauss-Newton matrix: it reduces the construction of $\bm{H}_{\lambda,\bm A}$ to $\tilde{p}$-dimensional Hessian--vector products (plus simple projections through $\bm A$).

\section{Accuracy analysis} 
\label{sec:accuracy}
We analyze the accuracy of ExpM-Quad (Algorithm \ref{alg:main_proj}) and its restricted variant in Section \ref{sec:projection}. For $D\in \mathcal{X}_N$, let 
\begin{equation*}
U_{\rm opt} = \max\limits_{\theta\in B^{(p)}_{R,\theta^*}} U(D, \theta)\ \text{ and } \ 
U_{\bm A, \rm opt} = \max_{{\xi} \in B^{(\tilde{p})}_{R,0}} U_{\bm A}(D, {\xi}) 
\end{equation*}
 be the optimal utility values in the full and restricted space, correspondingly. Although 
$U_{\rm opt}$ and $U_{\bm A, \rm opt}$ depend on $D$, we omit this dependence when no confusion arises.  The following tail bound controlling the utility loss relative to the optimum can be derived from standard arguments, see \cite[Theorem 3.11]{10.1561/0400000042}.
\begin{proposition}
\label{prop:accuracy}
For $t>0$, denote 
\begin{equation}
    \mathcal{B}_t = \{\theta \in B^{(p)}_{R,\theta^*} : U(D, \theta) \leq U_{\rm opt} - t\overline{\Delta U}\}
\end{equation}
and let \( \mathcal{B}_t^c = B^{(p)}_{R,\theta^*} \setminus \mathcal{B}_t\) to be its complement in $B^{(p)}_{R,\theta^*}$. If $\theta \sim \mathcal{M}_U(D)$, then 
\begin{align*}
\Pr[\theta \in \mathcal{B}_t] \leq \frac{|\mathcal{B}_t|}{|\mathcal{B}_{t/2}^c|} \cdot \exp\left( -\frac{\varepsilon t}{4} \right).
\end{align*}
Similarly, for the restricted mechanism \eqref{eq:distMUA}, denote \( \mathcal{B}_{t,\bm A} \) to be the subset of $B^{(\tilde{p})}_{R,0}$ satisfying \( U_{\bm A}(D, \xi) \leq U_{\bm A,\rm opt} - t\overline{\Delta U}  \) and  \( \mathcal{B}_t^c \) to be its complement. Then, if $\xi \sim \mathcal{M}_{U_{\bm A}}(D)$, 
\begin{align*}
\Pr[\theta \in \mathcal{B}_{t,\bm A}] \leq \frac{|\mathcal{B}_{t,\bm A}|}{|\mathcal{B}_{t/2,\bm{A}}^c|} \cdot \exp\left( -\frac{\varepsilon t}{4} \right).
\end{align*}
Here $|\cdot|$ denotes the Lebesgue measure. 
\end{proposition}

Proposition \ref{prop:accuracy} proves that our mechanism achieves near optimal utility score with high probability. However, for the restricted version, only the optimal score on the low-dimensional projected subspace can be reached. In \cref{prop:gap}, we establish a theoretical bound on the gap between the optimal utility scores in the full and restricted spaces. First, the following technical lemma is required.

\begin{lemma}\label{lem:beta_lower_tail}
Let $X\sim\mathrm{Beta}(\alpha,\beta)$.
For any $\eta\in(0,1)$, we have
\[
\mathbb{P}\!\left(X\le (1-\eta)\mathbb{E}[X]\right)\le \exp\!\left(-\frac{\alpha\eta^2}{4}\right).
\]
\end{lemma}

\begin{proof}
Set $t=(1-\eta)\mathbb{E}[X]$ and use the standard representation of $X$ in terms of Gamma distributions
\[
X = \frac{U}{U+V},
\]
where $U \sim \Gamma(\alpha,1)$ and $V \sim \Gamma(\beta,1)$ are independent.

The event $X\le t$ is equivalent to $(1-t)U-tV\le 0$, or $e^{-((1-t)U-tV)} \geq 1$. Hence, for any $s>0$, using Markov's inequality and the generating function of the Beta distribution,
\begin{align*}
\mathbb{P}(X \leq t)
&=\, \mathbb{P}\left[ e^{-((1-t)U-tV)} \geq 1 \right]\\
&\leq \mathbb{E}\left[ e^{-s((1-t)U-tV)} \right]
\\
 &= \mathbb{E}\left[ e^{-s(1-t)U)}\right] \mathbb{E}\left[ e^{stV}\right]\\
 &=  (1+s(1-t))^{-\alpha}(1-s t)^{-\beta}
\end{align*}
Applying estimates $\log(1+x)\ge x-\frac{x^2}{2}$ for $x\ge 0$ and
$-\log(1-x)\le x+\frac{x^2}{1-x}\le x+2x^2$ for $x\in(0,1/2)$
yields
\begin{align*}
\log\mathbb{P}(X\le t)  &\le \alpha \left(\frac{(s-st)^2}{2} - (s-st)\right)  + \beta(st + 2(st)^2) 
\\
& = (\alpha+\beta)st - \alpha s +  \frac{\alpha(s-st)^2}{2} + 2 \beta (st)^2 
\end{align*}
Replacing back $t= \alpha(1-\eta)/(\alpha + \beta)$ gives
\begin{align*}
\log\mathbb{P}(X\le t) &
\le - \eta\alpha s + \left[\frac{\alpha}{2} + 2\beta\frac{\alpha^2}{(\alpha + \beta)^2}\right] s^2 \\
&\le - \eta\alpha s + \alpha s^2 
\end{align*}
which is minimized at $s = \eta/2$, therefore 
\begin{align*}
\log\mathbb{P}(X\le t) \le -\frac{\alpha \eta^2}{4}.
\end{align*}
\end{proof}

\begin{proposition}\label{prop:gap}
Let \( \bm{A} \in \mathbb{R}^{p \times \tilde{p}} \) be drawn uniformly at random from the Stiefel manifold $\mathbb{V}_{p,\tilde{p}} := \{\bm A\in \mathbb{R}^{p\times \tilde{p}} | \bm A^\top \bm A = \bm{I}_{\tilde{p}}\}$. Let $
{z} := \bm{H}_\lambda^{-1/2} {g}
$ and $\kappa$ be the condition number of $\bm H_{\lambda}$. Assume $\tau < \|z\|^2/2 $. If 
\[
\tilde{p} \ge \max\left\{ {2\kappa p}\left(1- \dfrac{2\tau}{\|z\|^2}\right), 32\log\frac{1}{\gamma}\right\},
\]
then for any fixed $D\in \mathcal{X}_N$, with probability at least \( 1 - \gamma \), 
\[
 U_{\mathrm{opt}} - U_{\bm A, \mathrm{opt}}  \le \tau.
\]
\end{proposition}

\begin{proof} The utility function in the full space is given by
\[
U(D,\theta) \;=\; - ( \theta - \theta^*)^\top g - \tfrac{1}{2}\,( \theta - \theta^*)^\top \bm H_\lambda \,( \theta - \theta^*) .
\]
The optimal solution is achieved at
$
{\theta}_{\mathrm{opt}} = {\theta}^* -  \bm{H}_{\lambda}^{-1} {g}
$ with the optimal utility score being
$
U_{\mathrm{opt}} = \frac{1}{2} {g}^\top \bm{H}_{\lambda}^{-1} {g}, 
$

In the projected space, the restricted utility is
\[
U_{\bm{A}}(D, {\xi}) = -{\xi}^\top{g}_{\bm A}   - \tfrac{1}{2} {\xi}^\top  \bm {H}_{\lambda,\bm A}  {\xi}.
\]
The optimal \( {\xi}_{\mathrm{opt}} \) satisfies
\[
{\xi}_{\mathrm{opt}} = \bm{H}_{\lambda,\bm{A}}^{-1} {g}_{\bm{A}}, \quad
U_{\bm A, \mathrm{opt}} = \frac{1}{2} {g}_{\bm{A}}^\top \bm{H}_{\lambda, \bm{A}}^{-1} {g}_{\bm{A}}.
\]
The utility gap between the full and restricted spaces is thus
\begin{align*}
U_{\mathrm{opt}} - U_{\bm A, \mathrm{opt}} 
& = \frac{1}{2} {g}^\top \bm{H}_\lambda^{-1} {g} 
- \frac{1}{2} {g}^\top \bm{A} ( \bm{A}^\top \bm{H}_\lambda \bm{A} )^{-1} \bm{A}^\top {g}
\\
 & = \frac{1}{2} {g}^\top ( \bm{H}_\lambda^{-1} - \bm{A} ( \bm{A}^\top \bm{H}_\lambda \bm{A} )^{-1} \bm{A}^\top) {g}. 
\end{align*}

Let $P_{\bm A} := \bm A \bm A^\top$ denote the orthogonal projection onto the subspace $\operatorname{Im}(\bm A)\subset \mathbb{R}^p$. Define the orthogonal projection operator \( \Pi \) onto the subspace \( \operatorname{Im}(\bm{H}_\lambda^{1/2} \bm{A}) \subset \mathbb{R}^p \) as
\[
\Pi := \bm{H}_\lambda^{1/2} \bm{A} \left( \bm{A}^\top \bm{H}_\lambda \bm{A} \right)^{-1} \bm{A}^\top \bm{H}_\lambda^{1/2}.
\]
Then,
\begin{align*}
& U_{\mathrm{opt}} - U_{\bm A, \mathrm{opt}} 
 = \frac{1}{2} \left( \| {z} \|^2 - \| \Pi {z} \|^2 \right) 
\\
 \le\ & \frac{1}{2} \left( \| {z} \|^2 - \frac{1}{\kappa}\| P_{\bm A} {z} \|^2 \right) = \frac{\| z\|^2}{2}\bigl(1-\frac{\rho}{\kappa}\bigr)
\end{align*}
{with} $ X := \| P_{\bm A} {w} \|^2 \ \text{and}\   w :=  z/\| z\|  $. We have 
\begin{equation*}
\mathbb{P}\!\left(U_{\mathrm{opt}}-U_{\bm A,\mathrm{opt}}\le \tau\right)
\ge
 1-\mathbb{P}\!\left(X \le \kappa( 1-\frac{2\tau}{\|z\|^2})\right).
\end{equation*}
Since \( \bm{A} \) is sampled uniformly from the Stiefel manifold \( \mathbb{V}_{p,\tilde{p}} \), it can be shown that 
\[
X \;\sim\;  \operatorname{Beta}\left( \dfrac{\tilde{p}}{2}, \dfrac{p - \tilde{p}}{2} \right) .
\]
Set 
$ \eta := 1-\dfrac{\kappa p}{\tilde p}\left(1-\dfrac{2\tau}{\|z\|^2}\right)$, we have $\eta \in (1/2,1)$ due to the conditions of $\tilde{p}$ and $\tau$. Applying Lemma \ref{lem:beta_lower_tail} gives
\begin{align*}
\mathbb{P}\!\left(X \le \kappa \left( 1-\frac{2\tau}{\|z\|^2}\right)\right) &= 
\mathbb{P}\!\left(X\le (1-\eta)\frac{\tilde{p}}{p}\right) 
\\
& = \mathbb{P}\!\left(X\le (1-\eta)\mathbb{E}[X]\right)\\
&\le\;\exp\!\left(-\frac{\tilde p\,\eta^2}{8}\right) \\
&\le\;\exp\!\left(-\frac{\tilde p}{32}\right). 
\end{align*}
For $ \mathbb{P}\!\left(U_{\mathrm{opt}}-U_{\bm A,\mathrm{opt}}\le \tau\right) \;\ge\;1 - \gamma$, it suffices to choose $\tilde p$ such that
\[
\exp\!\left(-\frac{\tilde p}{32}\right)\le \gamma,
\quad\text{i.e.}\quad
\tilde p \ge {32}\log\frac{1}{\gamma}.
\]
\end{proof}
This proposition shows that restricting the sampling on a random $\tilde{p}$-dimensional subspace preserves the optimal utility up to an additive error $\tau$ with high probability. The ratio $\tilde{p}/p$ decreases when larger $\tau$ is allowed and vice versa. 

\section{Numerical experiments}
\label{sec:experiment}
We evaluate ExpM-Quad and compare it to DP-SGD. The pretrained model and non-private SGD results will be included for reference. We implement DP-SGD using Opacus library \cite{opacus}. For ExpM-Quad, we employ the restricted version \eqref{eq:distMUA} with various choices for subspace dimension and demonstrate that the original dimension can be significantly reduced with random projection strategy without sacrificing the performance of the algorithm. We also investigate the sensitivity of ExpM-Quad to the choice of $R$, and show how this parameter affects the performance. Gibbs sampling is employed to sample for the mechanism in all the tests. 

\subsection{Sinusoidal regression}
\label{sec:exp:sin_regression}
We start with a simple {sinusoidal regression} problem. 
The target function $\varphi\colon\R^5\to\R$ is of the form 
\[
\varphi(x) = a\,\sin(2\pi x_1) + b\,x_2 + c,
\]
where the inputs $x \in\R^5$, and only the first two coordinates carry signal. The function thus has a low-dimensional structure embedded in a higher-dimensional space. 
The pretraining and fine-tuning datasets both consider inputs $x \sim \mathcal{N}(0, I_5)$ but with target functions differ in amplitude $a$, slope $b$ and bias $c$, thereby introducing a mild {shift}. 
In particular, the pretraining data are generated with parameters 
$(a, b, c) = (1.0, 0.3, 0.25)$, 
while the fine-tuning data use 
$(a, b, c) = (0.9, 0.35, 0.25)$. 
To introduce additional domain shift between the pretraining and fine-tuning stages, we apply an affine transformation to the inputs. Instead of directly using the original inputs $x$, the function in fine-tuning stage is defined on the shifted and rescaled inputs 
\[
x_i \leftarrow 1.1(x_i + 0.1), \qquad i = 1,2.
\]
This synthetic setting, adapted from the sine regression test used in meta-learning~\cite{finn2017maml}, provides an analytically tractable testbed for examining the behavior of ExpM-Quad. 

For model architecture, we consider a Multilayer Perceptron (MLP) with two hidden layers of $10$ nodes each with ReLU activation functions, consisting of $p=181$ parameters in total. This MLP model $f(x,\theta)$ is pretrained on a dataset $D_0$ with $5000$ samples using mean squared error (MSE) loss. 
Pretraining is carried out for 2000 epochs with learning rate $10^{-2}$. 
The fine-tuning dataset $D$ is also of size $5000$ and used for testing differentially private fine-tuning with ExpM-Quad and the baselines. The final MSE loss of the pretrained MLP model on the pretraining dataset is $0.084$. When the pretrained model is applied to the fine-tuning dataset, the (zero-shot) loss value is $0.163$. 

For ExpM-Quad, we note that the sensitivity bounds \eqref{eq:sens_MS}-\eqref{eq:sens_CE} require suprema of residual and Jacobian spectral norms over the entire data distribution, whereas in practice we often only have access to empirical maxima over the observed dataset $D$. Therefore, we apply a conservative inflation factor of $1.1$ to the empirical bounds before computing sensitivity to account for unseen inputs beyond $D$. No regularization is applied to the quadratic loss (i.e., $\lambda=0$). The feasible parameter domain is the Euclidean ball
$B^{(p)}_{R,\theta^*}$. Sampling is carried out in a randomly chosen $\tilde{p}$-dimensional subspace of the
full parameter space $\R^p$ with various projection dimension $\tilde{p}$ and radius $R$. 
The performance of ExpM-Quad is evaluated and compared to DP-SGD over a range of privacy budgets
\[
\varepsilon \in \{0.1,\,1,\,2,\,5,\,10,\,50\}.
\] 

Figure~\ref{fig:mlp_sinreg} shows the fine-tuning performance of ExpM-Quad on the sinusoidal regression with fixed projected dimension $\tilde{p}=20$ and varied radius $R$. The best performance of ExpM-Quad is achieved at moderate $R$ ($0.1, 0.15$) where the method yields lower loss, passing the zero-shot loss value at $\varepsilon =1$ and approaching the non-private SGD baseline as $\varepsilon$ increases. On the other hand, too small a radius ($R=0.01$) over-restricts the feasible set, likely excluding the low-loss region and causing the loss to saturate at a suboptimal level regardless of $\varepsilon$. Too large a radius ($R=0.50$) inflates the sensitivity bound, requiring the mechanism to inject substantially more noise, which results in higher loss and large variance (shaded regions in \cref{fig:mlp_sinreg}), particularly when $\varepsilon$ is small. 
DP-SGD, in contrast, remains flat despite having a good start at $\varepsilon=0.1$.
We attribute this to the flatness of the loss landscape near $\theta^*$: the gradient is either too small or dominated by noise, even with small noise added (i.e., in large $\varepsilon$ case) to make effective steps. ExpM-Quad avoids this limitation because it directly samples from a distribution concentrated around the minimizer of the fine-tuning objective, without relying on iterative steps to navigate the loss landscape. This difference explains why ExpM-Quad exhibits a clear improvement as $\varepsilon$ grows while DP-SGD saturates.

\begin{figure}[!h]
\centering
\includegraphics[width = 0.45\textwidth]{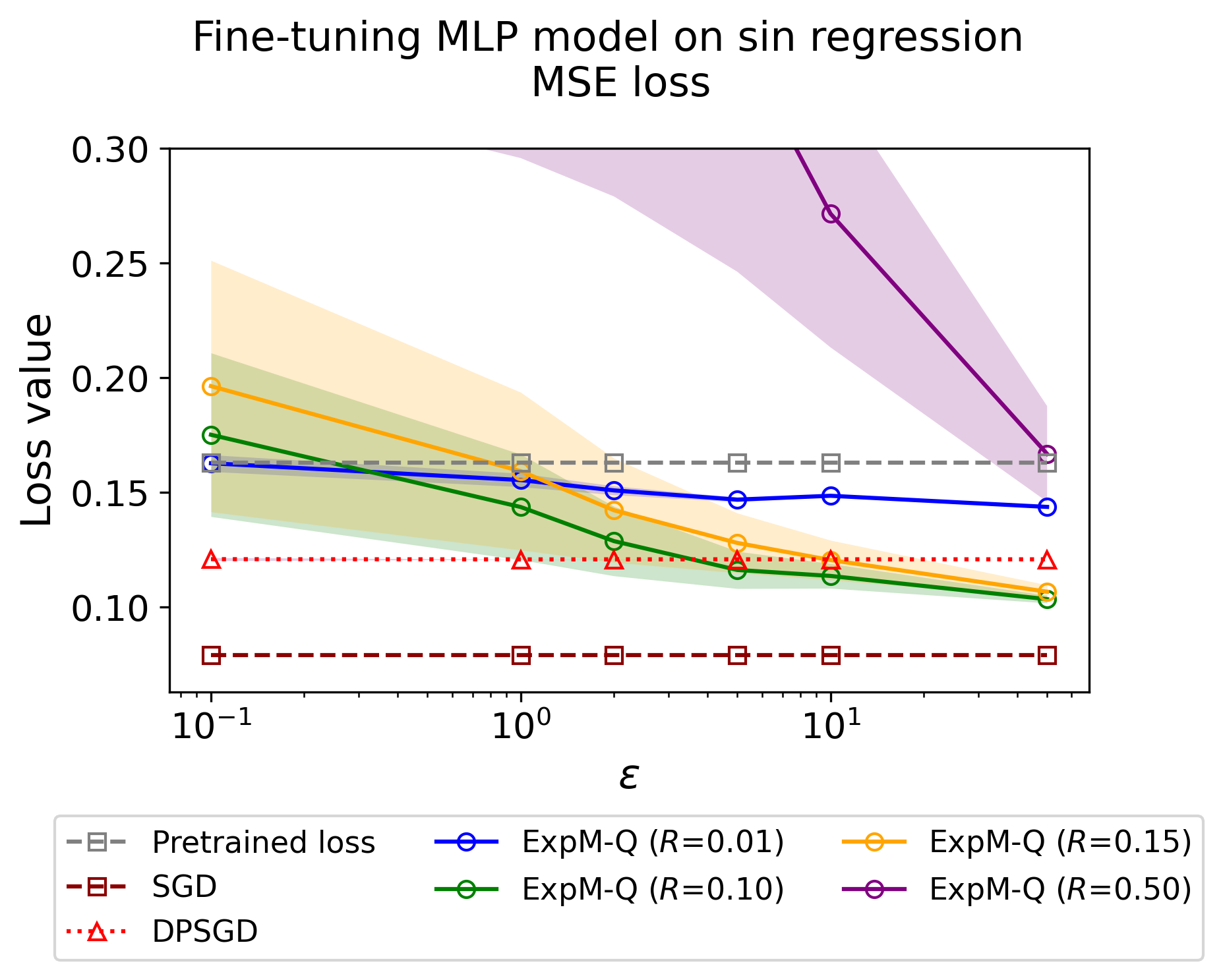}
\vspace{-.1in}
\caption{Fine-tuning an MLP on the sinusoidal regression problem
    with varied privacy budget $\varepsilon$. The ExpM-Quad projected dimension is
    fixed at $\tilde{p}=20$. For ExpM-Quad, each curve shows 
the mean loss value over 500 randomly drawn candidate models; shaded regions denote $\pm 1$ standard deviation.}
\label{fig:mlp_sinreg}
\end{figure}

Figure~\ref{fig:mlp_sinreg2} shows fine-tuning performance of ExpM-Quad 
with a fixed radius ($R=0.1$) and varied projected dimension $\tilde{p}$. 
All variants improve as $\varepsilon$ increases and converge toward 
the non-private SGD baseline at large $\varepsilon$. As expected, larger 
projected dimensions ($\tilde{p}=20, 40$) outperform smaller ones 
($\tilde{p}=5, 10$), since a higher-dimensional subspace better 
approximates the utility optimizer and imposes less 
restriction on the parameter sampling. This comes, however, at the cost 
of more expensive Gibbs or rejection sampling and Hessian computation. Notably, the 
performance gap is small and diminishes beyond $\tilde{p}=20$, despite the full parameter space having dimension $p=181$, suggesting that the fine-tuning objective is well captured in random low-dimensional subspaces near $\theta^*$.  This demonstrates
the efficiency of the subspace projection strategy in ExpM-Quad.

\begin{figure}[!h]
\centering
\includegraphics[width = 0.45\textwidth]{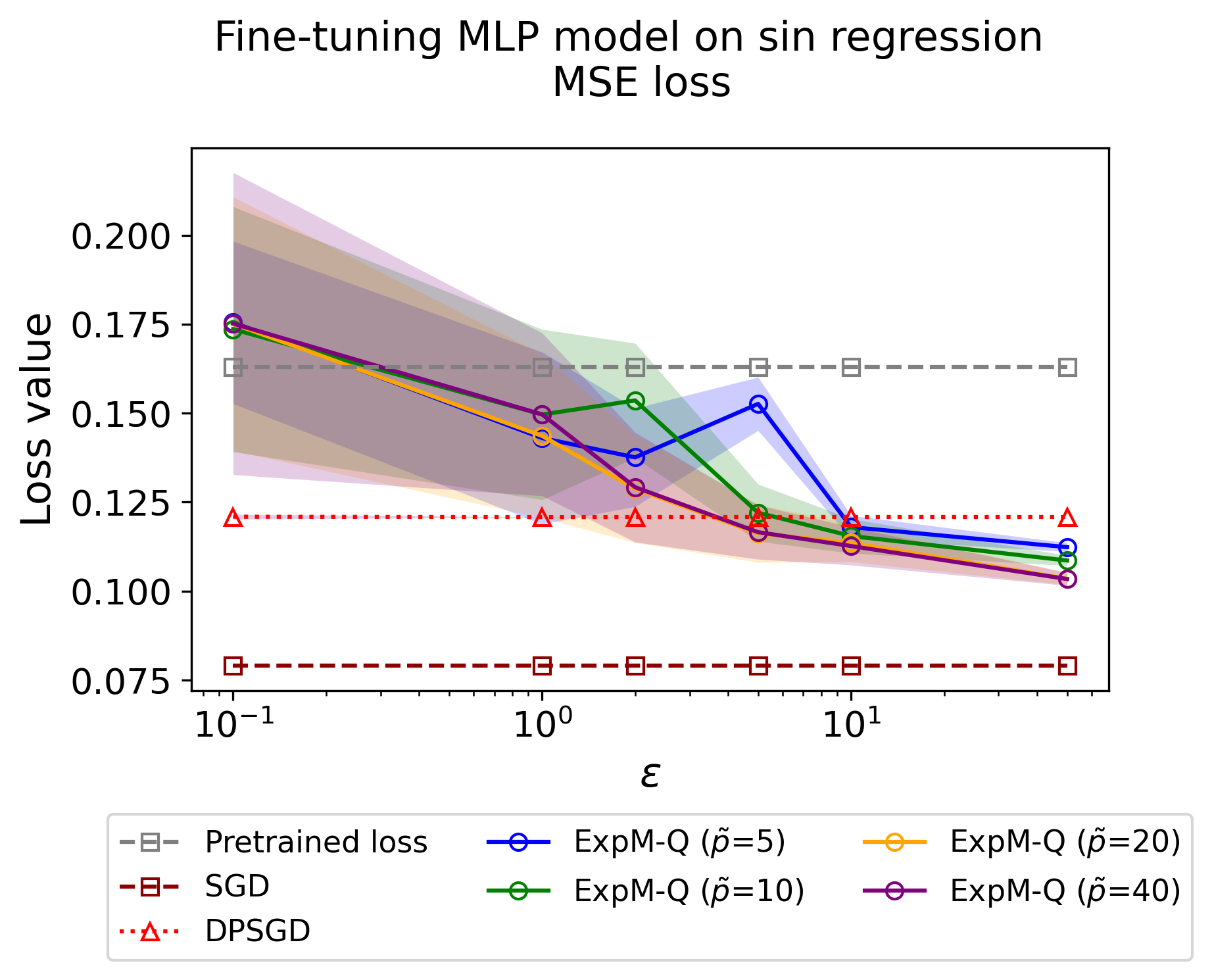}
\vspace{-.1in}
\caption{Fine-tuning an MLP on the sinusoidal regression problem
    with varied privacy budget $\varepsilon$. The ExpM-Quad feasible radius is 
    fixed at $R=0.1$. For ExpM-Quad, each curve shows 
the mean loss value over 500 randomly drawn candidate models; shaded regions denote $\pm 1$ standard deviation.}
\label{fig:mlp_sinreg2}
\end{figure} 

\subsection{MNIST classification}
\label{sec:exp:MNIST}

Next, we test a multiclass image classification task on the MNIST dataset \cite{lecun1998gradient}. 
All images are normalized in a standard manner with mean
$0.1307$ and standard deviation $0.3081$. To induce a distribution
shift between pretraining and fine-tuning datasets, we construct both clean and noisy
versions of the MNIST dataset. The noisy variant applies additive Gaussian noise with
zero mean and standard deviation $0.5$ prior to normalization. 

We consider a convolutional neural network (CNN) model, consisting of two convolutional blocks with ReLU activations and max-pooling, followed by an MLP layer producing 10-dimensional logits. This model has a total of $6{,}722$ parameters. The CNN model is pretrained on a clean subset of $20{,}000$ MNIST datapoints and subsequently fine-tuned on a noisy MNIST subset of $20{,}000$ datapoints. The SGD training takes $100$ epochs with batch size $1024$, learning rate $10^{-3}$, and weight decay $10^{-2}$. The MSE is used as the loss. Evaluation is performed on the corresponding clean and noisy test sets of $10{,}000$ samples. On the clean test set, the accuracy of pretrained model is $95.29\%$, which drops to $80.30\%$ when applied to the noisy fine-tuning set. For ExpM-Quad, we apply a regularization parameter $\lambda = 1.0$. Inflation factor of $1.1$ is applied to the empirical bounds of the residual and Jacobian norms.

In Figure~\ref{fig:mnist_cnn}, we show the fine-tuning performance of ExpM-Quad with fixed projected dimension ($\tilde{p}=400$) and varied radius $R$. Similar to the sinusoidal regression test, the best performance of ExpM-Quad is achieved at moderate $R$ ($0.1, 0.15$), where it surpasses the pretrained baseline at $\varepsilon=1$ and approaches the non-private 
SGD baseline at large $\varepsilon$. Small radius ($R=0.01$) again saturates at a suboptimal accuracy while large radius ($R=0.50$) suffers from stronger injected noise and larger sampling domain. 
DP-SGD outperforms ExpM-Quad for small $\varepsilon$ but remains nearly flat after 
$\varepsilon=1$, whereas ExpM-Quad exhibits clear monotone improvement and overtakes DP-SGD at large $\varepsilon$. 
For reference, we include the non-private SGD baseline with CE loss, which is more standard for classification. 
As expected, it outperforms SGD with MSE loss. 
ExpM-Quad is not competitive with CE loss in this test, probably because the quadratic approximation 
of the loss landscape is less accurate for CE than for MSE. However, ExpM-Quad with MSE loss still achieves comparable accuracy, demonstrating its effectiveness even when the loss function is not the standard choice to the task.
 
\begin{figure}[!h]
\centering
\includegraphics[width = 0.45\textwidth]{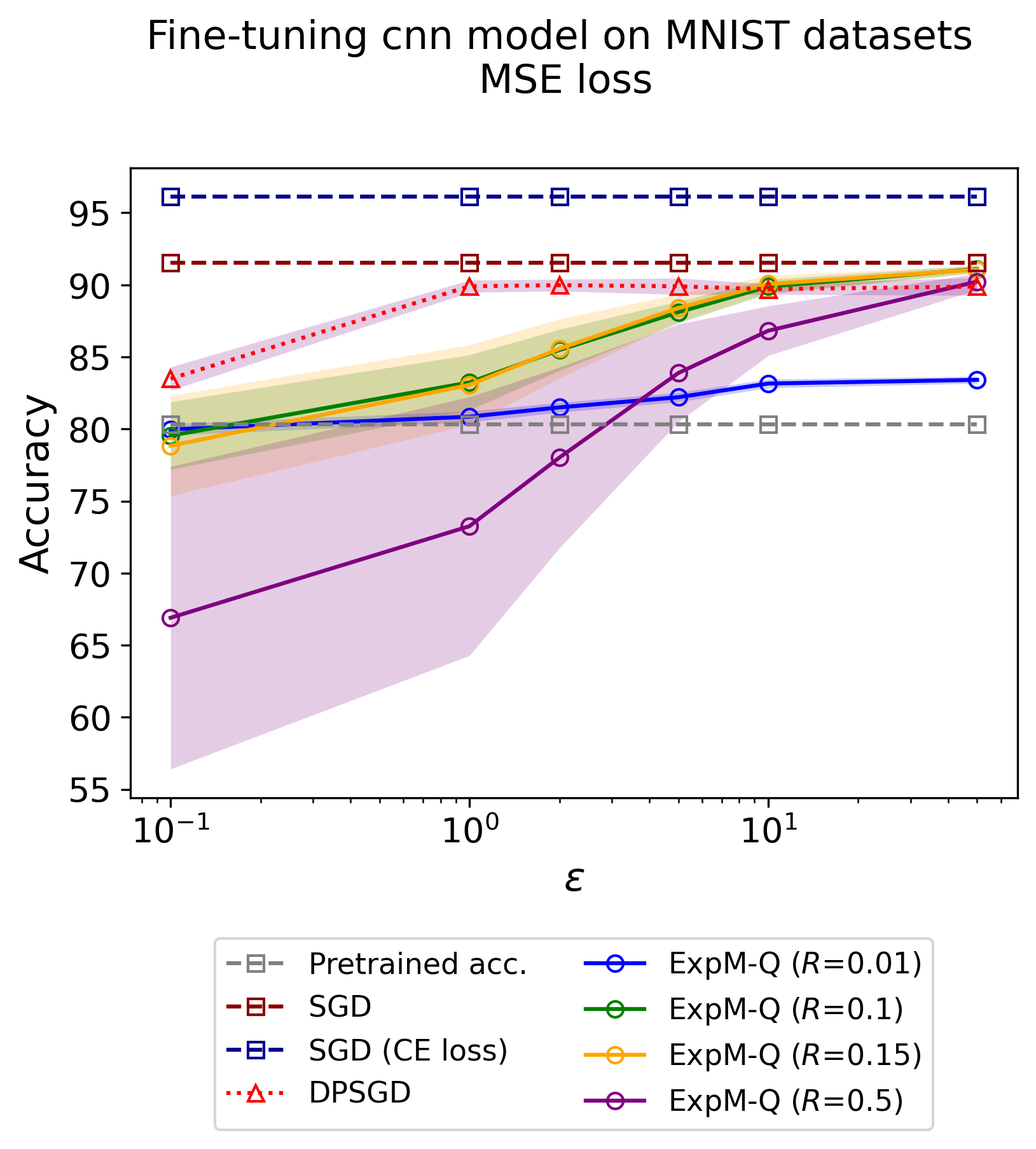}
\vspace{-.1in}
\caption{Fine-tuning a CNN model on MNIST classification with varied privacy budget $\varepsilon$. The ExpM-Quad projected dimension is fixed at $\tilde{p}=400$. For ExpM-Quad, each curve shows 
the mean accuracy over 500 randomly drawn candidate models; shaded regions denote $\pm 1$ standard deviation.}
\label{fig:mnist_cnn}
\end{figure} 

In Figure~\ref{fig:mnist_cnn2}, we show the fine-tuning performance of ExpM-Quad with fixed radius $R=0.1$ and varied subspace dimension $\tilde{p}$. Here, all variants perform comparably at small $\varepsilon$, but larger projected dimensions ($\tilde{p}=200, 400$) outperform smaller ones ($\tilde{p}=50, 100$) at $\varepsilon \geq 5$. At $\tilde{p}=400$, 
ExpM-Quad reaches the non-private SGD baseline, suggesting that a $400$ dimensional subspace is sufficient to capture the full space landscape near the pretrained $\theta^*$.

\begin{figure}[!h]
\centering
\includegraphics[width = 0.45\textwidth]{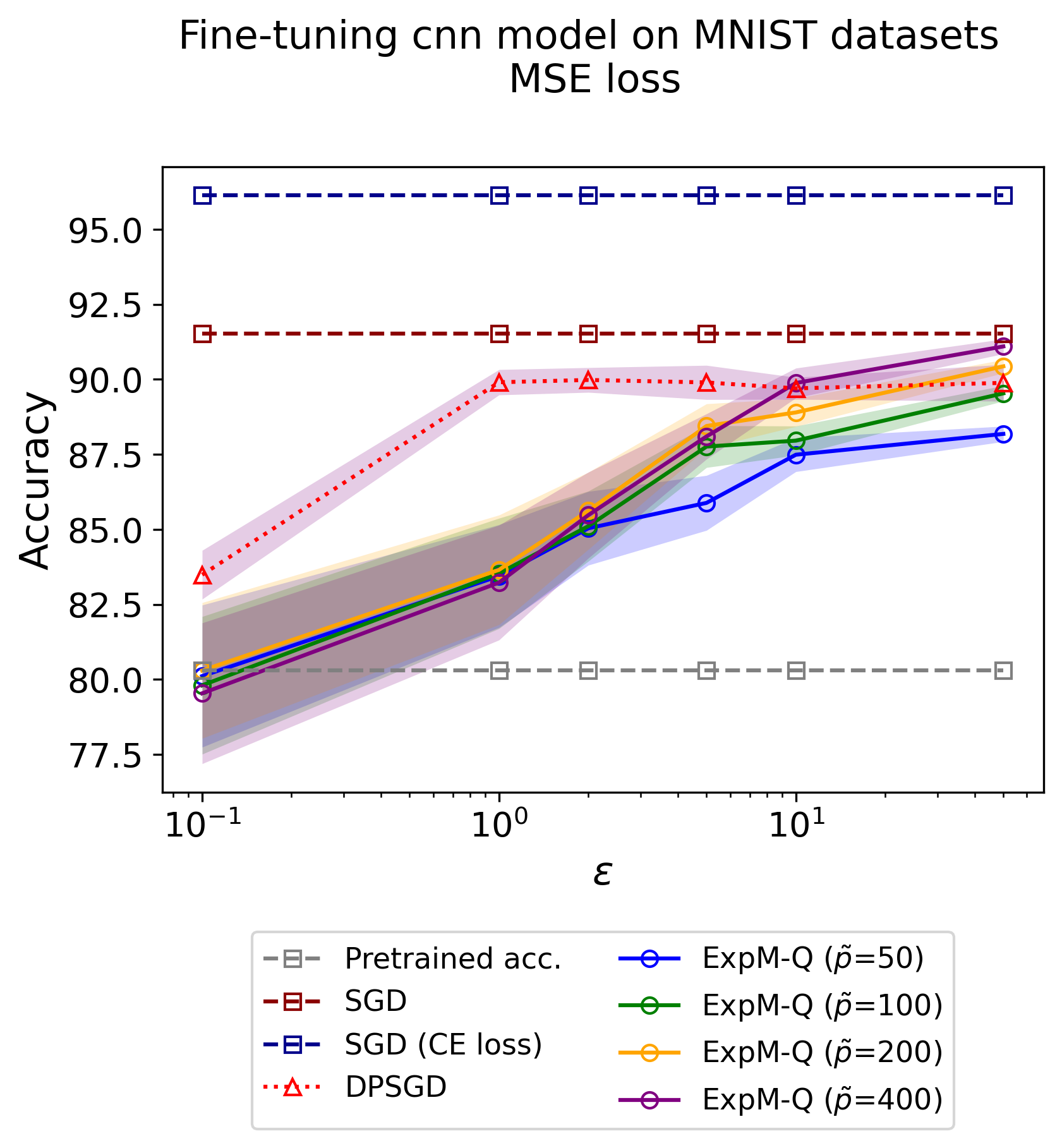}
\vspace{-.1in}
\caption{Fine-tuning a CNN model on MNIST classification with varied privacy budget $\varepsilon$. The ExpM-Quad feasible radius is 
    fixed at $R=0.1$. For ExpM-Quad, each curve shows 
the mean accuracy over 500 randomly drawn candidate models; shaded regions denote $\pm 1$ standard deviation.}
\label{fig:mnist_cnn2}
\end{figure}

\subsection{MIMIC-IV Mortality Prediction}
\label{sec:exp:MIMIC}

We consider a binary classification task in which the objective is to predict patient mortality from benchmark clinical datasets. The machine learning model is first pretrained on the eICU Collaborative Research Database \cite{eicu} and subsequently fine-tuned on MIMIC-IV \cite{mimic-iv}. We use the first 48 hours ICU admission window for the mortality prediction task. The eICU and MIMIC-IV datasets are harmonized using the BlendedICU preprocessing library \cite{Oliver2023Introducing}, then the MIMIC-IV-Data-Pipeline \cite{gupta2022extensive} is used to convert the harmonized data into a format suitable for machine learning training.
The eICU dataset comprises $183{,}631$ adult ICU stays collected from more than $200$ hospitals across the United States, of which $173{,}605$ are labeled as non-mortal and $10{,}026$ as mortal. Its multi-center nature introduces high variability in patient populations, clinical practice patterns, and hospital environments. The MIMIC-IV dataset contains $69{,}964$ adult ICU stays from a single academic medical center, including $7{,}777$ mortal cases. Although smaller in size, MIMIC-IV offers a more coherent and high-fidelity clinical setting, making it well suited for domain-specific fine-tuning and evaluation.

Both datasets exhibit severe class imbalance. To ensure effective evaluation using accuracy metric, we conduct experiments on balanced subsets of each dataset. Specifically, we pretrain the model on a subset of $15{,}000$ eICU cases, with $50\%$ mortal and $50\%$ non-mortal samples, and fine-tune on a subset of $15{,}000$ MIMIC-IV cases, similarly balanced across outcome classes. The original input features consist of a $2{,}403$-dimensional sparse vector encoding diagnosis, procedure, and medication codes. To reduce input dimensionality, we apply random forest–based feature selection on the eICU training data to identify the most important features for mortality prediction. 200 top-ranked features are retained and used during pretraining and fine-tuning.

We consider the linear model $f(x) = Wx + b$, pretrained for $1{,}000$ epochs with a learning rate $0.05$, weight decay $0.1$, and MSE loss. In Figure~\ref{fig:eicu_linear}, with fixed projected dimension ($\tilde{p}=20$), moderate $R$ ($0.1$) provides the best performance of ExpM-Quad, surpassing the pretrained accuracy at $\varepsilon \ge 5$. DP-SGD is probably guided by very small or noisy update steps and remains flat across all $\varepsilon$. Figure \ref{fig:eicu_linear2} shows that random projection onto $20$ or $40$ dimensional subspaces is sufficient, as the performance gap of ExpM-Quad between $\tilde{p}=20$ and $\tilde{p}=200$ is minimal.    

\begin{figure}[!h]
\centering
\includegraphics[width = 0.45\textwidth]{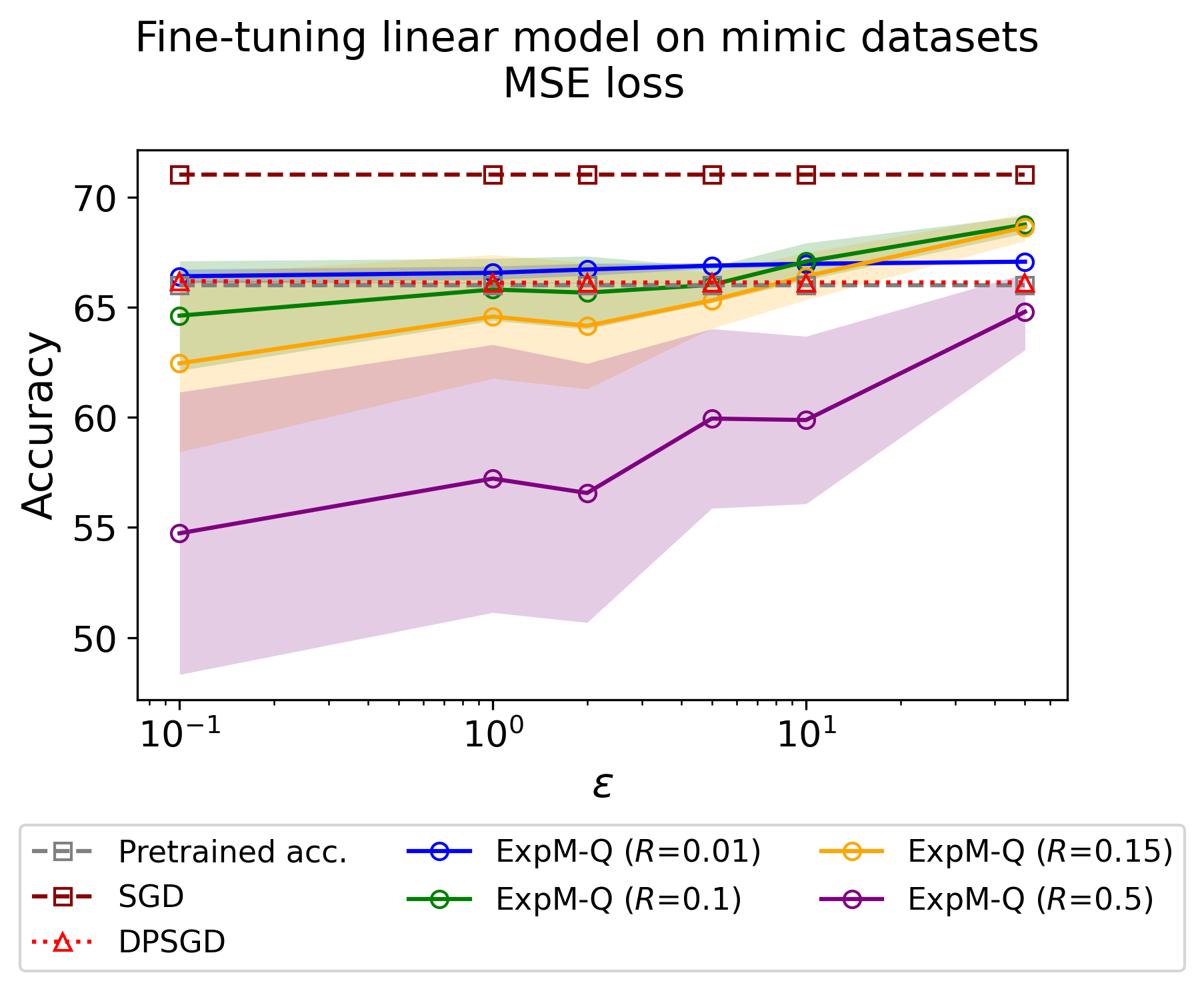}
\vspace{-.1in}
\caption{Fine-tuning a linear model on MIMIC-IV mortality prediction with varied privacy budget $\varepsilon$. The ExpM-Quad projected dimension is
    fixed at $\tilde{p}=20$.  }
\label{fig:eicu_linear}
\end{figure} 

\begin{figure}[!h]
\centering
\includegraphics[width = 0.45\textwidth]{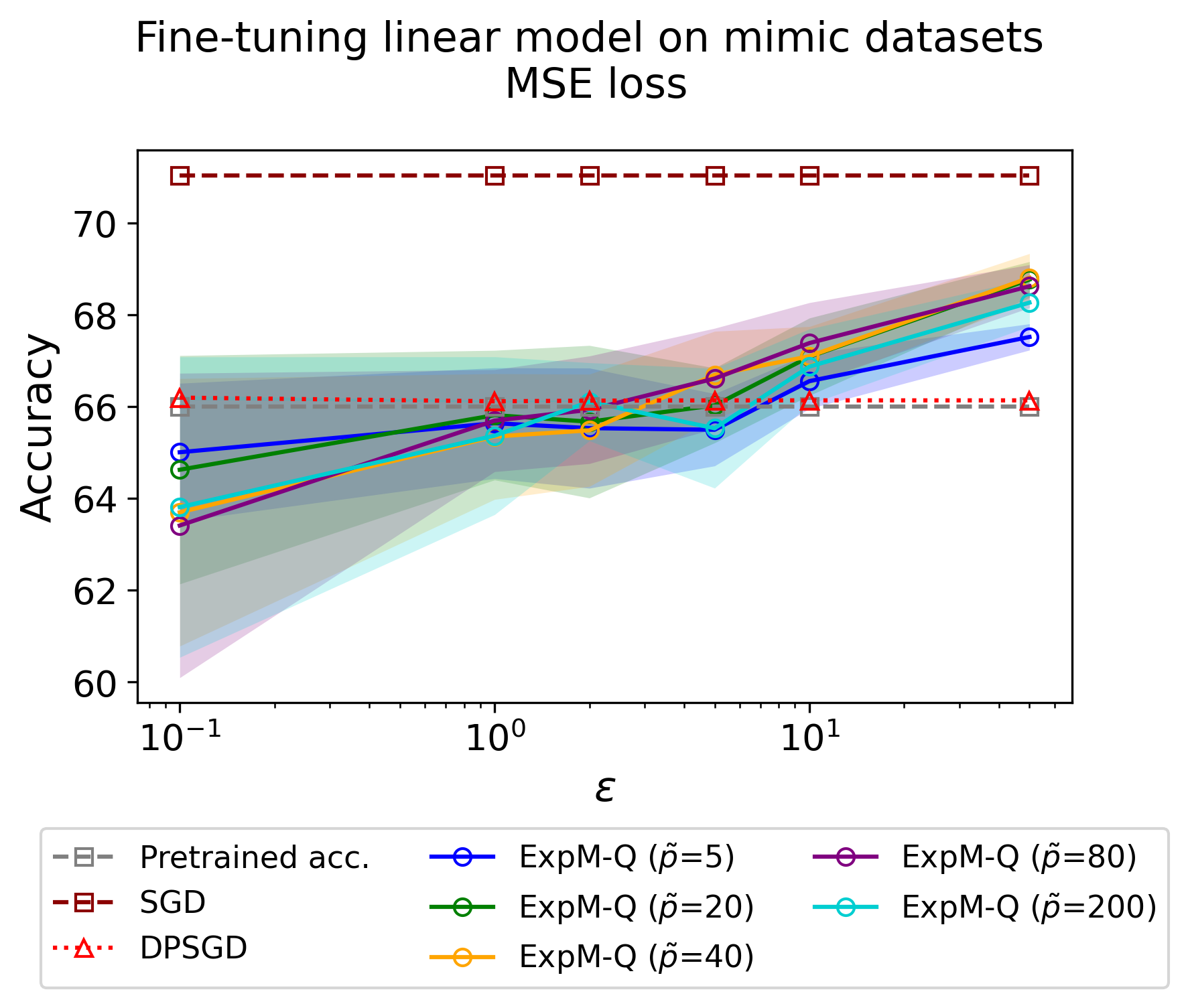}
\vspace{-.1in}
\caption{Fine-tuning a linear model on MIMIC-IV mortality prediction with varied privacy budget $\varepsilon$. The ExpM-Quad feasible radius is 
    fixed at $R=0.1$ }
\label{fig:eicu_linear2}
\end{figure} 

\section{Conclusions}
\label{sec:conclusion}

We introduce ExpM-Quad, a differentially private fine-tuning method 
based on the exponential mechanism that exploits quadratic approximation of the 
fine-tuning loss function near the pretrained parameters. By constructing a 
quadratic loss/utility function, ExpM-Quad allows for both rigorous sensitivity 
analysis and tractable sampling via a truncated Gaussian distribution, 
overcoming two fundamental obstacles that had limited the practical 
adoption of exponential mechanisms. 
To scale the approach to high-dimensional models, we propose a random projection strategy that restricts sampling to a low-dimensional subspace. We prove that this projection preserves near-optimal utility 
with high probability and verify numerically that random subspaces of dimension orders 
of magnitude smaller than the full parameter space suffice to capture 
the essential loss landscape near the pretrained parameters.

In three benchmark experiments -- sinusoidal regression, MNIST 
classification, and MIMIC-IV mortality prediction -- ExpM-Quad achieves 
clear improvement in performance as the privacy budget 
$\varepsilon$ grows, approaching the non-private SGD baseline at large 
$\varepsilon$. A comparison with DP-SGD highlights two advantages 
of our method. First, while 
ExpM-Quad requires an overhead cost for computing Hessian information, once it is done, a single sampling chain can quickly produce hundreds of candidate models, while DP-SGD yields only one model per run. Second, 
DP-SGD is path-dependent and can stagnate in flat loss regions where 
gradient steps are too small or too noisy to make meaningful progress. 
ExpM-Quad does not face this limitation because it directly samples from a 
distribution concentrated around the minimizer of the approximate loss function, without relying on iterative updates to navigate the loss landscape.

Several directions remain for future work. First, the current approach 
uses a randomly chosen projection subspace; employing data-driven selection of subspaces, such as using the top eigenvectors of the Hessian on a pretraining or reference dataset, may further improve performance. Second, a more thorough theoretical study on the quality of quadratic approximation, as well as the interplay 
between the feasibility radius and 
the fine-tuning performance would provide better understanding of when ExpM-Quad is most effective. 
Third, scaling ExpM-Quad to 
more advanced models with millions of 
parameters remains an important challenge, particularly given the cost of Hessian 
computation at scale, which warrants 
further investigation. 

\section*{Acknowledgments}
The authors thank Dr.~Heidi Hanson for acquiring funding supporting this work.  
The authors thank Dr.~Heidi Hanson and Dr.~John Gounley for the discussions on defining the overarching goal of this work and on the formulation of the MIMIC-IV mortality prediction experiment. 
This material is based upon work supported by the U.S.~Department of Energy, Office of Science, Office of Advanced Scientific Computing Research under Award Number DE-SC-ERKJ422.


\bibliographystyle{IEEEtran}
\bibliography{reference}

\begin{IEEEbiographynophoto}
{Hoang A. Tran} received the M.S. degree in Mathematics from the Université d’Orléans, Orléans, France, in 2008, and the Ph.D. degree in Applied Mathematics from the University of Pittsburgh, Pittsburgh,
PA, USA, in 2013. He is currently a mathematician with Data Analysis and Machine Learning Group, Computer Science and Mathematics Division, Oak
Ridge National Laboratory, Oak Ridge, TN, USA.
His research interests include compressed sensing, machine learning, high-dimensional approximations and numerical solution of partial differential equations.
\end{IEEEbiographynophoto}

\begin{IEEEbiographynophoto}
{Jorge Ramirez} is a Colombian applied mathematician, educator, and researcher. He earned his M.Sc. and Ph.D. degrees in Mathematics from Oregon State University and was a Richard S. Pierce Fellow in Mathematics at The University of Arizona. From 2022 to 2026, he served as a research staff member in the Systems and Decision Sciences Group at Oak Ridge National Laboratory. He is currently an Associate Professor in the Department of Mathematics at Universidad Nacional de Colombia. His research focuses on uncertainty quantification, probability, and stochastic processes, with applications to the natural sciences and computation.  
\end{IEEEbiographynophoto}

\begin{IEEEbiographynophoto}
{Jiayi Wang} received the B.E. degree in electrical engineering from the University of Electronic Science and Technology of China in 2017 and the Ph.D. degree in electrical engineering from the University of Utah in 2024. She is currently a Postdoctoral Researcher at Oak Ridge National Laboratory. Her research interests include federated learning, differential privacy, synthetic data generation, and distributed optimization. 
\end{IEEEbiographynophoto}

\begin{IEEEbiographynophoto}
{Alberto Bocchinfuso} received his BS in Computer Engineering from University of Calabria (Italy) in 2016, his MS in Automation and Control Engineering from Politecnico di Milano (Italy) in 2019 and his PhD in Applied Mathematics from Case Western Reserve University, in Cleveland, OH, in 2023.
He served as a Postdoctoral Research Associate in Oak Ridge National Laboratory before joining CINECA (Bologna, Italy) where he currently serves as HPC Specialist.
His research interest include inverse problems, data science, mathematical modeling and machine learning.
\end{IEEEbiographynophoto}

\begin{IEEEbiographynophoto}
{Chris Stanley} received his Ph.D. degree in Polymer Science and Engineering from the University of Massachusetts, Amherst, MA, USA, in 2004. He is a Staff Research Scientist in the Advanced Computing for Health Sciences section, Computational Sciences and Engineering Division, Oak Ridge National Laboratory, Oak Ridge, TN, USA. His research interests include machine learning for scientific and clinical applications.
\end{IEEEbiographynophoto}

\begin{IEEEbiographynophoto}
{M.~Paul Laiu} received his Ph.D. degree in Electrical and Computer Engineering from University of Maryland College Park, MD, USA, in 2016.
He is a Staff Mathematician in the Multiscale Methods and Dynamics Group at Computer Science and Mathematics Division, Oak Ridge National Laboratory, Oak Ridge, TN, USA. 
His research interest includes scientific machine learning, surrogate modeling, iterative solvers, and numerical schemes for various partial differential equations, with focuses on the design, development, and analysis of mathematical tools that accelerate the simulation and learning of multiscale systems.

\end{IEEEbiographynophoto}

\vfill\pagebreak

\end{document}